\documentclass[9.5pt,journal,compsoc]{IEEEtran}
\ifCLASSINFOpdf
\else
\fi
\usepackage{subfig}
\usepackage{array}
\usepackage{amsmath} 
\usepackage{graphicx,wrapfig,lipsum}
\usepackage{rotating}
\usepackage{amsmath, amssymb, amsthm}
\usepackage{stmaryrd}
\usepackage{tikz}
\usepackage{verbatim}
\usepackage{url}
\usepackage[utf8]{inputenc}
\usepackage[english]{babel}
\newtheorem{theorem}{Theorem}[]

\newtheorem{lemma}[]{Lemma}
\usepackage{multicol}
\usepackage{xr-hyper}

\newtheorem{assumption}{Assumption}[]
\usepackage{algorithm}
\usepackage{algorithmicx}
\usepackage{algpseudocode}
\usepackage{algcompatible}
\usepackage{scalerel}
\usepackage{multicol}
\usepackage{setspace}
\newtheorem{definition}{Definition}
\usepackage[noadjust]{cite}
\usepackage{bbm}
\usepackage[normalem]{ulem}
\usepackage{color}
\usepackage{dsfont}
\usepackage{mathtools}
\usepackage{enumitem}
\usepackage{bm}
\usepackage[nolist,nohyperlinks]{acronym}
\renewcommand\IEEEkeywordsname{Keywords}
\usetikzlibrary{automata}
\usetikzlibrary{arrows}
\allowdisplaybreaks
\usepackage{tikz}
\usetikzlibrary{automata,positioning}
\bstctlcite{IEEEexample:BSTcontrol}
\usepackage{verbatim}
\DeclareMathOperator*{\argmax}{arg\,max}

\newcommand{\comm}{\hfill $\triangleright$ {}}

	\begin{acronym}
    \acro{4G}{fourth generation}
    \acro{5G}{fifth generation}
    \acro{AoA}{angle of arrival}
   	\acro{AoD}{angle of departure}
   	\acro{AP}{access point}
    \acro{BCRLB}{Bayesian CRLB}
    \acro{BS}{base stations}
	\acro{CDF}{cumulative density function}
    \acro{CF}{closed-form}
    \acro{CRLB}{Cramer-Rao lower bound}
    \acro{EI}{Exponential Integral}
    \acro{eMBB}{enhanced mobile broadband}
    \acro{FIM}{Fisher Information Matrix}
    \acro{GPS}{global positioning system}
    \acro{GNSS}{global navigation satellite system}
    \acro{HetNets}{heterogeneous networks}
    \acro{LOS}{line of sight}
    \acro{MAB}{multi-armed bandit}
    \acro{MBS}{macro base station}
    \acro{MEC}{mobile-edge computing}
    \acro{MIMO}{multiple input multiple output}
    \acro{mm-wave}{millimeter wave}
    \acro{mMTC}{massive machine-type communications}
    \acro{MS}{mobile station}
    \acro{MVUE}{minimum-variance unbiased estimator}
    \acro{NLOS}{non line-of-sight}
    \acro{OFDM}{orthogonal frequency division multiplexing}
    \acro{PDF}{probability density function}
    \acro{PGF}{probability generating functional}
    \acro{PLCP}{Poisson line Cox process}
    \acro{PLT}{Poisson line tessellation}
    \acro{PLP}{Poisson line process}
    \acro{PPP}{Poisson point process}
    \acro{PV}{Poisson-Voronoi}
    \acro{QoS}{quality of service}
    \acro{RAT}{radio access technique}
    \acro{RL}{reinforcement-learning}
    \acro{RSSI}{received signal-strength indicator}
    \acro{BS}{base station}
   	\acro{SINR}{signal to interference plus noise ratio}
    \acro{SNR}{signal to noise ratio}
    \acro{TS}{Thompson Sampling}
    \acro{TS-CD}{TS with change-detection}
    \acro{UCB}{upper confidence bound}
	\acro{ULA}{uniform linear array}
	\acro{UE}{user equipment}
 	\acro{URLLC}{ultra-reliable low-latency communications}
    \acro{V2V}{vehicle-to-vehicle}    
\end{acronym}

\begin{document}
\title{A Change-Detection Based Thompson Sampling Framework for Non-Stationary Bandits}
\author{Gourab~Ghatak, {\it Member, IEEE}
\thanks{G. Ghatak is with the Department of ECE at IIIT Delhi. Email: gourab.ghatak@iiitd.ac.in.}
}
\IEEEtitleabstractindextext{
\begin{abstract}
    We consider a non-stationary two-armed bandit framework and propose a change-detection based Thompson sampling (TS) algorithm, named TS with change-detection (TS-CD), to keep track of the dynamic environment. The non-stationarity is modeled using a Poisson arrival process, which changes the mean of the rewards on each arrival. The proposed strategy compares the empirical mean of the recent rewards of an arm with the estimate of the mean of the rewards from its history. It detects a change when the empirical mean deviates from the mean estimate by a value larger than a threshold. Then, we characterize the lower bound on the duration of the time-window for which the bandit framework must remain stationary for TS-CD to successfully detect a change when it occurs. Consequently, our results highlight an upper bound on the parameter for the Poisson arrival process, for which the TS-CD achieves asymptotic regret optimality with high probability. Finally, we validate the efficacy of TS-CD by testing it for edge-control of radio access technique (RAT)-selection in a wireless network. Our results show that TS-CD not only outperforms the classical max-power RAT selection strategy but also other actively adaptive and passively adaptive bandit algorithms that are designed for non-stationary environments.
\end{abstract}
\vspace{-1cm}
}

\vspace{-2cm}
\maketitle

\section{Introduction}
The \ac{MAB} problem is a sequential decision-making framework where an agent chooses one or multiple actions, from a set of actions, based on the feedback of rewards from the previous choices. The \ac{MAB} framework has found applications in the field of randomized clinical-trials~\cite{villar2015multi}, online recommendation systems~\cite{li2016collaborative}, computational advertisement~\cite{buccapatnam2017reward}, and wireless communications~\cite{rahman2019beam}.
In the basic \ac{MAB} setting, a decision maker plays (or chooses) one of the K independent arms (choices), and obtains a corresponding reward. The player repeats this experiment in a series of time-slots. Each of the K-independent arms is characterized by a reward distribution, which are not necessarily identical, and are unknown to the player. The goal of any bandit algorithm is to minimize the difference between the total rewards obtained by the player and the highest expected reward. To facilitate this, the player keeps a belief or value function associated to each arm at each time-slot, as the \ac{MAB} algorithm evolves, and plays an arm accordingly. This inherently has an exploration-exploitation dilemma associated to it, i.e., the player has to make a decision on whether to play the arm which currently has the highest belief/value, or explore other arms to update their belief/value functions. 

{\bf Related Work:}
In case the reward distributions of the arms remain stationary, several bandit algorithms, e.g., \ac{UCB}~\cite{contal2013parallel}, have been proven to perform optimally. Recently, the \ac{TS} algorithm, which was first introduced in~\cite{thompson1933likelihood}, has gotten considerable interest. Several studies have shown the efficiency of the \ac{TS} algorithm, albeit empirically, e.g., see~\cite{granmo2010solving}. In particular, for applications such as display advertising and online article recommendation, it outperforms other classical algorithms~\cite{chapelle2011empirical}. However, mathematical characterization of the \ac{TS} algorithm is in general difficult due to its randomized nature, unlike \ac{UCB}, which facilitates exploration by an additional term. In this regard, Agarwal and Goyal~\cite{agrawal2012analysis} have recently derived bounds on the regret of the \ac{TS} algorithm.

On the contrary, in case the reward distributions are non-stationary, bandit algorithms loose mathematical tractability. To keep track of the changing environments, in literature, two approaches have been proposed: i) passively adaptive, and ii) actively adaptive. Passively adaptive policies remain oblivious towards the time of changes and give more weights to the recent rewards with an aim to decrease the effect of the changes. Garivier and Moulines~\cite{garivier2008upper} have considered a scenario where the distribution of the rewards remain constant over epochs and change at unknown time instants. They have analyzed the theoretical upper and lower bounds of regret for the discounted \ac{UCB} (D-UCB) and sliding window \ac{UCB} (SW-UCB). 
Gupta et al.~\cite{gupta2011thompson}, extending the idea to Bayesian methods, have proposed dynamic \ac{TS}. By assuming a Bernoulli bandit environment where the success probability evolves as a Brownian motion, the authors suggest to decay the effect of past observations in the posterior distribution of the arm being updated. This is done by applying an exponential filtering to the past observations. Raj and Kalyani~\cite{raj2017taming} have formulated a Bayesian bandit algorithm for non-stationary environments. Derived from the classical \ac{TS} algorithm, their proposed method - discounted TS (dTS) - works by discounting the effect of past observations. Besbes {\it et al.}~\cite{besbes2014stochastic} have developed a near optimal policy, REXP3. The authors have established lower bounds on the performance of any non-anticipating policy for a general class of non-stationary reward distributions. In principle, the REXP3 algorithms runs the classical EXP3 algorithm in different blocks by resetting the parameters in the beginning of each block. {The optimal exploration exploitation tradeoff for non-stationary environments were studied by the authors in~\cite{besbes2019optimal}}.

On the contrary, actively adaptive bandit algorithms track the dynamic environment and take certain actions (e.g., restarting the algorithm) when the variation of the environment is detected. Hartland et al.~\cite{hartland2006multi} have considered dynamic bandits with {abrupt changes} in the reward generation process, and proposed an algorithm called {Adapt-EvE}. It uses a Page-Hinkley statistical test (PHT) based change point detection technique and utilizes a meta bandit formulation for exploration-exploitation dilemma. However, they did not provide a theoretical study about these procedures, and their performance evaluation is mainly empirical. PHT  has also been used to adapt the window length of SW-UCL~\cite{srivastava2014surveillance}. However, the regret bounds of Adapt-EvE and adaptive SW-UCL are still open problems.
{Two works which are the closest to ours are the ones by Yu {\it et. al}~\cite{yu2009piecewise} and Cao {\it et al.}~\cite{cao2019nearly}, where the authors sense a change in the environment by detecting a change in the empirical means of the rewards of the arms. However, in both these works, the authors assume a fixed number of changes withing an interval to derive the respective bounds. We make no such stringent restriction; on the contrary, we assume random number of changes, and statistically characterize the probability that the changes occur more frequently than what the algorithm can keep track of. Moreover, we provide tunable parameters of false-alarm and missed detection probabilities in the change detection framework, which is missing from~\cite{yu2009piecewise, cao2019nearly}. Another strong assumption in~\cite{yu2009piecewise} is that their algorithm needs to query and observe the past rewards of some unpicked arms, which is a very important assumption. In fact, in most practical systems, such as the one considered in our case-study, this may be impractical. Accordingly, our proposal makes no such assumption. Finally, following the change-detection framework, in contrast to \cite{yu2009piecewise, cao2019nearly}, we employ the \ac{TS} algorithm which has been empirically shown to outperform index-based policies, but notoriously challenging to mathematically characterize.}

In summary, passively adaptive algorithms give more tractable theoretical guarantees. However, it has been demonstrated experimentally that actively adaptive algorithms outperform passively adaptive ones~\cite{mellor2013thompson}. 
Following this line of work, in this paper, we consider a two-armed bandit setting and propose a variant of the \ac{TS} algorithm leveraging a mean-estimation based change-detection framework that actively adapts the algorithm once a change is detected. The contributions of the paper are as follows.

\subsection{Contributions and Organization}
\begin{itemize}
    \item We {adopt} a mean-estimation based change detection framework, where, using the reward returns of the previous plays of the arms, the framework maintains two sets of sequences at any time-step $n$: i) the rewards for the last $n_T$ plays, termed as the {\it test distribution} and ii) the rewards from the play $n - n_T - N$ to $n-n_T$, termed as the {\it estimate distribution}. The framework detects a change when the mean of the above two sequences differ by more than a threshold $\Delta_C$. Following this, we derive the bounds on $n_T$ and $N$ and study the conditions under which the proposed algorithm is able to detect a change. Accordingly, we mathematically characterize the probability of false-alarm and the probability of missed-detection of the proposed algorithm.
    \item Based on the change-detection framework, we propose a \ac{TS} algorithm called \ac{TS-CD} which refreshes the \ac{TS} parameters when a change is detected. To the best of our knowledge, no other work in literature has investigated actively adaptive \ac{TS} algorithms, theoretically. Then, we study the probability of the failure of the framework, by taking into account the i) probability that changes occur faster than the detection framework, ii) the error in estimation of the means, and iii) the probability of missed detection. Finally, we derive the upper bound on the regret of the proposed \ac{TS-CD} algorithm and study the probability with which the regret follows this bound.
    \item To validate the efficacy of the algorithm, we employ it in a wireless network equipped with decentralized control for \ac{RAT} switching between sub-6GHz and \ac{mm-wave} bands, where the rewards are characterized by using stochastic geometry. We show that the proposed framework outperforms the classical max-power band-selection strategy and also other bandit algorithms that are designed for tackling non-stationary environments.
\end{itemize}

\vspace{-0.3cm}
\section{The Two-Armed Bandit Setting}
\label{sec:Bandit}
Let us consider a two-armed bandit framework, with arms $a_i$, where $i \in \{1,2\}$. At each time step $n$, the player selects an arm and observes a corresponding reward $R_{a_i}(n)$. We assume that the reward of arm $a_i$ has a normal distribution\footnote{Gaussian distribution has been previously considered in bandit literature, e.g.,~\cite{srivastava2014surveillance}.}, i.e., $R_{a_i}(n) \sim \mathcal{N}\left(\mu_i(n),\sigma^2\right)$, with unknown, non-stationary mean and known, fixed variance $\sigma^2$. Additionally, we consider a lower bound on the minimum difference between the mean reward of the two arms at any instant of time, i.e.,
\begin{align}
  \left| \mu_i(n) - \mu_j(n) \right| \geq \Delta_\mu, \quad {i \neq j}, \forall n. \label{eq:mindelmu}
\end{align}
Furthermore, we assume a lower bound on the minimum difference, {$\Delta_m$} between the mean rewards of the same arm across the time slots when the changes occur, i.e.,
{
\begin{align}
  \left| \mu_i(n) - \mu_i(k) \right| \in \{0\} \cup [\Delta_m,\infty), \quad n \neq k. \label{eq:mindel}
\end{align}
}
In this paper, we will consider the case when $\mu_i(n)$s are piece-wise stationary. {The values of $\mu_i(n)$ are assumed to change at the unknown time-instants, $T_{C_l}$, where $l = 1, 2, \ldots$, with $T_{C_0}$ assumed to be at $n = 0$.} The change instants follow a Poisson arrival process with parameter $\lambda_C$, and accordingly, the time-window between the changes, i.e., {$T_{C_{l+1}} - T_{C_{l}}$ is exponentially distributed: 
\begin{align}
        \mathbb{P}\left(T_{C_{l+1}} - T_{C_{l}} \leq k\right) = 1 - \exp\left(-\lambda_C k\right).
        \label{eq:change} 
\end{align}
The mean of both the arms are assumed to change simultaneously at $T_{C_l}$.
}

{\bf: Regret}
An algorithm $\pi$ to solve the \ac{MAB} problem needs to decide at any time step $n$, which arm to play (say $a_{\pi}(n)$) and obtains the reward $R_{a_\pi}(n) \sim \mathcal{N}\left(\mu_{a_\pi}(n), \sigma^2\right)$, based on the arms chosen and the rewards obtained in time steps $0, 1, 2, \ldots, n-1$.
Let at a time-step $n$, the $a_i$ arm be the optimal arm, i.e., $\mu^*(n) = \mu_i(n) >  \mu_j(n)$, $i \neq j$. Then, the regret of the algorithm after playing $T$ rounds is given by:
\begin{align}
    \mathcal{R}(T) = \sum_{n=0}^T \left(\mu^*(n) - \mu_{a_{\pi}}(n)\right). \nonumber 
\end{align}
The aim of \ac{MAB} algorithms is to develop policies $\pi$ which bounds the regret in an expected sense, i.e., $\mathbb{E} \left[\mathcal{R}(T)\right]$. For that, we present the change-detection algorithm in the following section that tracks the non-stationary environment.

{\bf\ac{TS-CD}:}
To actively detect the change when it occurs we propose \ac{TS-CD} in Algorithm 1, where $T_N$ is derived in \eqref{eq:TN}, $n_T$ is derived in \eqref{eq:nT}, and $\Delta_C$ is defined in \eqref{eq:DelC}. In particular, $n_T$ is the number of time-slots after a change at $T_{C_i}$ required to detect the change. Furthermore, $T_N$ refers to the minimum number of time-slots for which the \ac{MAB} framework remains stationary after the detection of the change. In the next section, we will mathematically characterize $n_T$ and $T_N$ in terms of the system parameters.
\begin{algorithm}[!h]
\begin{algorithmic}[1]
 \STATE ${B_{1i}} \gets 0; {B_{2i}} \gets 0; \quad \forall i\in\{1, 2, \ldots, K\} $ \comm{initializing parameters}
    \STATE $t \gets 0$
    \WHILE{$1$}
        \STATE $\theta_i \sim \beta({B_{2i}}+ 1, {B_{1i}} -{B_{2i}} +1)$ \comm{draw from Beta dist.}
        \STATE $a_j \gets a_i | \theta_j = \max(\theta_i)$ \comm{choose the better arm}
        {
        \STATE $R_\pi(n) \gets  \frac{R_{a_j}(n) - \mu_{\min} + \sigma Q^{-1}\left(\epsilon_b\right)}{2\sigma Q^{-1}\left(\epsilon_b\right) + \mu_{\max} - \mu_{\min}}$ \comm{play the chosen arm evaluate the Bernoulli parameter}
        }
        \STATE $R^*$ = Bern $(R_\pi(n))$ \comm{Bernoulli Trial}
        \STATE ${B_{1j}}  \gets {B_{1j}} + 1 - R^*$ \comm{update the beta distribution}
        \STATE ${B_{2j}}  \gets {B_{2j}} + R^*$ \comm{update the beta distribution}
        \STATE $\mathcal{S}_j = \mathcal{S}_j \cup \{R_j\}$
        \STATE count = count + 1  \comm{update counter}
        \STATE $n \gets n+1$
        \IF{ (count $\geq T_N$)} \comm{change-detection phase starts}
            \STATE $a_j \gets \argmax_i \{\hat{\mu}_i(T_N)\}$ \comm{select the current best arm}
            \STATE $\mathcal{S}_j = \mathcal{S}_j \cup \{R_j\}$ \comm{repeatedly play current best arm}
            \IF{(4)}  \comm{detect the change}
                \STATE ${B_{1i}} \gets 0; {B_{2i}} \gets 0; \quad   i\in\{1, 2\}$ \comm{refresh parameters}
                \STATE count $\gets 0$; \comm{reset the counter}
            \ENDIF
        \ENDIF
    \ENDWHILE
\end{algorithmic}
 \caption{\ac{TS-CD}}
 \label{algo:ts}
\end{algorithm}

In \ac{TS-CD}, first we initialize the parameters of the beta distribution for the two arms, $B_{1i}$ and $B_{2i}$, for $i \in \{1,2\}$ respectively with 0. This is consistent with the classical \ac{TS} algorithm, and is due to the fact that we do not assume any prior information about the mean of the rewards. Then, at each time-step, we sample from the beta distribution for the two arms and play the arm which returns the larger sample (say arm $a_j$). Consequently, we receive a reward $R(n)  = R_{a_j}(n)$. {It must be noted that due to our assumption of Gaussian distributed rewards, the range of $R(n) \in (-\infty, \infty)$. However, for the sake of tractability and to employ the \ac{TS} algorithm of~\cite{agrawal2012analysis}, we map $R(n)$ to the range $[0,1]$. For both the arms, let us define two boundary points $L$ and $U$ as $
    \mathcal{P}(L \leq R_{a_j(n)} \leq U) \geq 1 - 2\epsilon_b, \quad \forall j \in \{1,2\}, $
where, $\epsilon_b$ is an arbitrary small positive number that defines the boundaries $L$ and $U$. Naturally, the values of $U$ and $L$ depend on the maximum and minimum values of $\mu_j$, respectively. Accordingly, let us define:
\begin{align}
    \mu_{\max} &= \max {\mu_j(n)}, \quad, \forall j, n \quad \mbox{and,} \nonumber \\
    \mu_{\min} &= \min {\mu_j(n)}, \quad, \forall j, n \nonumber 
\end{align}
as the maximum and the minimum mean rewards of the arms across all the time instants. Thus, the values of $U$ and $L$ can be calculated as:
\begin{align}
    U &= \sigma Q^{-1}\left(\epsilon_b\right) + \mu_{\max} \nonumber \\
    L &= \mu_{\min} - \sigma Q^{-1}\left(\epsilon_b\right)  \nonumber
\end{align}
Accordingly, we define:
\begin{align}
    R_{\pi}(n) = \frac{R_{a_j}(n) - L}{U - L} \nonumber
\end{align}
which ensures that $0 \leq R_{\pi}(n) \leq 1$ with a probability $1 - 2 \epsilon_b$.} and perform a Bernoulli trial with a success probability $R_{a_j}$. This step is similar to the algorithm presented in~\cite{agrawal2012analysis}. Following the result of the Bernoulli trial, the parameters of the beta distribution for the belief of the arm $a_j$ is updated, i.e., {$B_{1j}$} is augmented by 1 in case the Bernoulli trial results in a failure, and the parameter {$B_{2j}$} is augmented by 1 in case the the Bernoulli trial results in a success. The sequence of the rewards for each arm $j$ is stored in a reward set $\mathcal{S}_j$. We assume that the distribution of the rewards remain stationary for $T_F \geq T_N + n_T$ time steps after every time a change occurs. After $T_N$ time-steps (as tracked by the variable {\it count}) since the last detected change, the change-detection part of the algorithm initiates (step 13 - 20). For count $\geq T_N$, until the change is detected, the player repeatedly chooses the arm (say arm $a_j$) which had the best empirical mean at the time-step when count $=T_N$ (i.e., after $T_N$ time-steps since the last detected change). {Note that the step 14, and consequently the step 15, is necessary since our \ac{MAB} framework considered is based on \ac{TS}, which is a randomized algorithm as compared to classical index-based policies. Thus, the absence of steps 14 and 15 would result in a non-zero (albeit small) probability of playing the non-optimal arm in a particular stationary regime, which would limit the change detection efficacy.} The set $\mathcal{S}_j$ is {then} updated in each step for $a_j$ {(i.e., the optimal arm), until the change is detected}. From the set $\mathcal{S}_j$, we create two subsets: i) the test sequence: the rewards for the last $n_T$ plays of arm $a_j$, and ii) the estimate sequence: the rewards for the last $N$ rewards before the previous $n_T$ plays. In other words, in case the cardinality of $\mathcal{S}_j$ is $L$, the test-sequence consists of $\mathcal{S}_j(L-n_T : L)$ and the estimate sequence consists of $\mathcal{S}_j(L - n_T - N_j : L-n_T-1)$. The change is detected when the mean of the test sequence ($\mu_\text{test}$) differs from the mean of the estimate sequence ($\hat{\mu}_{i}$) by more than $\Delta_C$. Once a change is detected, the parameters are reset (step 13).

Before proceeding to the mathematical characterization of the algorithm, we make the following important assumption on the number of plays of a \ac{TS} algorithm for the stationary two-armed bandit.
\begin{assumption}
For the classical \ac{TS} algorithm in the stationary two-armed bandit framework, given a set of reward distribution for the two arms (i.e., for a given $\mu_1$ and $\mu_2$), the number of plays of the two arms is approximately equal to the respective mean number of plays. In other words, across different realizations of the \ac{TS} algorithm for a given $\mu_1$ and $\mu_2$, the number of plays of the arm $i$, $N_i$, follows $N_i \approx \mathbb{E}[N_i]$.
\label{assumption:mean}
\end{assumption}
Assumption 1 is needed for Lemma~\ref{lem:T_N}, and we may relax this assumption under the guarantee that the optimal arm is played $T_N$ times, as explined later.

\vspace{-0.3cm}
\section{Mathematical characterization of \ac{TS-CD}}
\label{sec:math}
As mentioned in the last section, we detect a change when the empirical mean of the test sequence differs from the mean of the estimate sequence by a factor greater than $\Delta_C$. In other words, we detect a change, when for an arm $a_i$:
\begin{align}
   \left| {\underbrace{\frac{1}{n_T} \sum_{p = L - n_T}^{L} \mathcal{S}_i(p)}_{\mu_\text{test}}} - {\underbrace{\frac{1}{N_i} \sum_{q = n- n_T - N_i}^{n - n_T} \mathcal{S}_i(q)}_{\hat{\mu}_i}} \right|  \geq  \Delta_C. \label{eq:criterion} 
\end{align}
Here $\mu_\text{test}$ refers to the estimate of the test distribution and $\hat{\mu}_i$ refers to the estimate of the mean of the reward of arm $a_i$ obtained from the estimate sequence. Let us assume that a change occurs at time $T_{C_l}$ and is detected at a time $n = T_{D_l} > T_{C_l}$. For evaluating $\mu_\text{test}$, out of the $n_T$ elements of the test sequence, let us assume that $n_1$ samples $({X_1}, X_2, \ldots, X_{n_1})$ are from the distribution $X_i \sim \mathcal{N}\left(\mu_i(T_{C_l} - 1), \sigma^2\right)$ and $n_2$ samples $({Y_1}, Y_2, \ldots, {Y_{n_2})}$ are from the distribution $Y_i \sim \mathcal{N}\left(\mu_i(T_{C_i}+1), \sigma^2\right)$.. In order to simplify the notation, we denote $\mu_i(T_{C_l}-1)$ by $\mu_x$ and $\mu_i(T_{C_l}+1)$ by $\mu_y$. In other words, the mean of the optimal arm $a_i$ changes from $\mu_x$ to $\mu_y$ at $T_{C_i}$, with the condition that $|\mu_x - \mu_y| \geq \Delta_m$. Furthermore, without loss of generality, we assume that\footnote{The case where $\mu_y > \mu_x$ follows similarly.} $\mu_y < \mu_x$. Thus, the detection criterion is:
\begin{align}
   \left| \frac{1}{n_T} \left( \sum_{p = n- n_T}^{n-n_T + n_1} X_p  +  \sum_{p = {n-n_T + n_1}}^{n} Y_p \right)\right.  \nonumber \\ \left.
    -\frac{1}{N_i}  \sum_{q = n- n_T - N_i}^{n - n_T}  X_q \right| \geq \Delta_C.
   \label{eq:CD}
\end{align}
For the mathematical characterization of the change detection framework, let us first informally introduce the following events, which we will make mathematically precise later in this section. {We outline here that these events are listed in the following order since the events $E_{2}$ and $E_3$ become relevant if and only if $E_1$ does not occur.}
\begin{itemize}
    \item $E_1$: Two consecutive changes in the arms occur too often for the detection to keep track. {In other words, for the \ac{TS-CD} algorithm to precisely estimate the mean of the optimal arm, and then detect a change when it occurs, the \ac{MAB} framework needs to be stationary for a period of time characterized by the change frequency bound, as discussed in Lemma 3.}
    \item $E_2$: The estimate of the mean of arm $i$ is not accurate. {That is, given that $E_1$ does not occur, and the \ac{MAB} framework experiences sustained periods of stationarity before a change, the \ac{TS-CD} algorithm $N_T$ number of samples to accurately estimate the mean reward of the optimal arm, as characterized in Lemma 1.}
    \item $E_3$: The change detection framework is not able to detect the change. {In other words, given that $E_1$ and $E_2$ do not occur, the \ac{TS-CD} algorithm needs sufficient number ($n_T$) of samples to detect a change when it occurs, as characterized in Lemma 2.}
\end{itemize}
Considering $E_2$, let us first define the accuracy of our estimate of $\hat{\mu}_i$.
\begin{definition}
$\hat{\mu}_i$ is said to be {well-localized} if for some $\epsilon \geq 0$, 
\begin{align}
    \mu_i - \epsilon \leq \hat{\mu}_i \leq {\mu}_i + \epsilon . \nonumber 
\end{align}
\end{definition}
The Chernoff-Hoeffding bound states that for a large value of $N_i$, this occurs with a high probability, i.e.,
\begin{align}
    {
    \mathbb{P}\left(\hat{\mu}_i < \mu_i + \epsilon \quad \cup \quad  \hat{\mu}_i > \mu_i - \epsilon\right) \geq 1 - 2\exp\left( - N_i \epsilon^2\right). \nonumber
    }
\end{align}
Accordingly, in the following lemma, we will characterize the number of plays of a stationary \ac{TS} algorithm so that sufficient plays of the optimal arm occurs so as to have $\mu_i$ well-localized.
\begin{lemma}
Under Assumption~\ref{assumption:mean}, in a time-window where the \ac{MAB} framework remains stationary, with $a_i$ being the optimal arm, the minimum number of plays of the \ac{TS-CD} framework for the estimate of the mean of the best arm, $\mu_i$ to be well-localized with probability greater than $1 - p_{loc}$ is given by:
\begin{align}
    T_N = \frac{-40}{\Delta_\mu^2} \mathcal{W}\left(- \exp\left(\frac{-40}{\Delta_\mu^2}\left(\frac{1}{\epsilon}\ln \left(\frac{1}{p_{loc}}\right) - \frac{48}{\Delta_\mu^4}\right)\right)\frac{\Delta^2_\mu}{40}\right),
    \label{eq:TN}
\end{align}
where $\mathcal{W}(\cdot)$ is the Lambert-W function.
\label{lem:T_N}
\end{lemma}
\begin{proof}
The number of plays of the optimal arm, i.e., $N_i$, for the following condition to hold true:
$
    \mathbb{P}\left(\hat{\mu}_x \leq \mu_x - \epsilon\right) < p_{loc},
$
is naturally, $N_i \geq \frac{1}{\epsilon} \ln \frac{1}{p_{loc}}$. Note that we are only considering the lower limit of the mean of the optimal arm since the sub-optimal arm has a mean lower than the optimal arm. Now, we have to characterize the number of plays of a stationary \ac{TS} algorithm so that $N_i$ plays of the optimal arm takes place, so as to have the estimate of its mean well-localized.
We know that in $T$ plays of the TS algorithm, the expected number of plays of the sub-optimal arm $a_j$ is bounded as~\cite{agrawal2012analysis}:
$
    \mathbb{E}\left[N_j\right] \leq \frac{40 \ln T}{\Delta^2} + \frac{48}{\Delta^4} + 18. 
$
Consequently, given Assumption 1, we want the number of plays of the optimal arm to satisfy the following criterion:
\begin{align}
    &T - \frac{40 \ln T}{\Delta^2} + \frac{48}{\Delta^4} + 18 \geq \frac{1}{\epsilon} \ln \frac{1}{p_{loc}}, \nonumber \\
    \implies & \exp\left(T\right) T^{\frac{-40}{T}} \geq \exp\left(\frac{1}{\epsilon} \ln \frac{1}{p_{loc}} - \frac{48}{\Delta^4} - 18\right), \nonumber \\
    \implies & T \geq T_N =  \nonumber \\
    &\frac{-40}{\Delta_\mu^2} \mathcal{W}\left(- \exp\left(\frac{-40}{\Delta_\mu^2}\left(\frac{1}{\epsilon}\ln \left(\frac{1}{p_{loc}}\right) - \frac{48}{\Delta_\mu^4}\right)\right)\frac{\Delta^2_\mu}{40}\right).  \nonumber 
\end{align}
\vspace{-0.5cm}
\end{proof}
Lemma~\ref{lem:T_N} characterizes the minimum number of time-steps for which, if the two-armed bandit framework remains stationary, the mean of the optimal arm is well-localized. In Fig.~\ref{fig:T_N} we plot the bound on $T_N$ with varying $p_{loc}$ for different values of $\Delta_\mu$. Naturally, with increasing value of $p_{loc}$ or $\epsilon$, the constraint on the accuracy of the estimate of $\mu_i$ becomes less stringent and accordingly, the lower bound on $T_N$ decreases.  Similarly, the lower bound on $T_N$ also decreases with {increasing} $\Delta_\mu$. This is due to the fact that if the the mean rewards of both the arms are well-separated, the \ac{TS} framework results in a larger number of plays of the first arm. This results in a better estimate of $\mu_i$ with a lower value of $T_N$.
Now, once the mean of the optimal arm is well-localized, we have to detect a change, if it occurs. Accordingly, let us consider the case of failure of detection of change, given that $\hat{\mu}_i$ is well-localized, i.e., we consider $E_3$. The question that we will try to answer is: what should be the minimum number of samples $n_T$ so as to have sufficient confidence of detection.
\begin{lemma}
Given that $\mu_i$ is well localized at $\mu_x$, for a false-alarm probability $\mathcal{P}_F$, to limit the probability of failure of change-detection to $\mathcal{P}_M$, the number of samples in the test set is:
\begin{align}
     n_T = \frac{1}{\Delta_\mu}\left( \sqrt{\ln \frac{1}{\mathcal{P}_M}} + \sigma Q^{-1}\left(\mathcal{P}_F\right) \right). 
     \label{eq:nT}
\end{align}
\end{lemma}
\begin{proof}
\begin{align}
    &\mathbb{P}\left(\text{Failure} \right)  = \mathbb{P}\left(\frac{1}{n_T} \sum_{p = n - n_T}^n Y_p > \hat{\mu}_x - \Delta_C \right), \nonumber \\
    & \leq  \mathbb{P}\left(\frac{1}{n_T} \sum_{p = n - n_T}^n Y_p > {\mu}_x - \epsilon - \Delta_C \right), \nonumber \\
    & \leq \mathbb{P}\left(\frac{1}{n_T} \sum_{p = n - n_T}^n Y_p > \mu_y + \Delta_m - \epsilon - \Delta_C \right), \nonumber \\
    & \leq \exp\left(- \left(\Delta_m - \epsilon - \Delta_C\right)^2 n_T\right). \label{eq:failure} 
\end{align}
Next we describe the choice of $\Delta_C$. The answer to this lies in the tolerable maximum false-alarm rate. Let the acceptable probability of false alarm be $\mathcal{P}_{F}$. False alarm occurs when a change is detected even though all the samples of the test distribution are from $X$.
\begin{align}
    &\mathbb{P}\left(\text{False Alarm}\right) = \mathbb{P}\left(\frac{1}{n_T} \sum_{p = n - n_T}^n X_p > \hat{\mu}_x - \Delta_m \right), \nonumber \\
    & \leq \mathbb{P}\left(\frac{1}{n_T} \sum_{p = n - n_T}^n X_p > {\mu}_x + \epsilon - \Delta_m \right) \nonumber \\
    &= Q\left(\frac{\sqrt{n_T}\left(\epsilon + \Delta_m\right)}{\sigma}\right). \nonumber 
\end{align}
Equating this to $\mathcal{P}_F$, we get:
\begin{align}
    \Delta_C = \frac{\sigma Q^{-1}\left(\mathcal{P}_F\right)}{\sqrt{n_T}} - \epsilon . \label{eq:DelC}
\end{align}
Substituting this in \eqref{eq:failure}, we have:
\begin{align}
    \mathbb{P}\left(\text{Failure with } n_T \text{ samples}  \right) \nonumber \\
    \leq \exp\left(- \left(\Delta_m - \frac{\sigma Q^{-1}\left(\mathcal{P}_F\right)}{\sqrt{n_T}}\right)^2 n_T\right).
\end{align}
Now equating this to $\mathcal{P}_M$, we get:
\begin{align}
    \exp\left(- \left(\Delta_m - \frac{\sigma Q^{-1}\left(\mathcal{P}_F\right)}{\sqrt{n_T}}\right)^2 n_T\right) = \mathcal{P}_M, \nonumber \\
    \implies n_T = \frac{1}{\Delta_m}\left( \sqrt{\ln \frac{1}{\mathcal{P}_M}} + \sigma Q^{-1}\left(\mathcal{P}_F\right) \right). 
\end{align}
This completes the proof.
\end{proof}
In Fig.~\ref{fig:n_T} we plot the lower bound on $n_T$ with respect to the probability of missed detection, for different allowable false alarm probabilities, $\Delta_m$, and $\sigma$. Naturally, with a less stringent constraint on the probability of missed detection, the bound on $n_T$ decreases. Interestingly, the effect of $\Delta_m$ and $\sigma$ on $n_T$ is much more, as compared to the probability of missed detection and the probability of false alarm.

Now that we have characterized the lower bounds on $n_T$ and $T_N$, we study the bounds on $\lambda_A$ so that the changes occur less frequently than the \ac{TS-CD} algorithm is able to localize the optimal arm and track the changes. That is, we focus on the event $E_1$.
\begin{lemma}
To limit the probability of the frequency of change to $p_{change}$, the bound on the value of $\lambda_A$ is:
\begin{align}
    \lambda_A \leq  \frac{1}{n_T + T_N} \ln\left(\frac{1}{1-p_{change}}\right).
\end{align}
\end{lemma}
This directly follows from the exponential distribution of the inter-change times \eqref{eq:change}. In Fig.~\ref{fig:lambda_A} we plot the bound on $\lambda_A$ with respect to $p_{change}$ for different values of $p_{loc}$. Naturally, for less stringent requirement of $p_{change}$, the upper-bound on $\lambda_A$ is higher. On the other hand, as $p_{loc}$ increases, the values of $T_N$ decreases, and the bound on $\lambda_A$ increases. As seen in Fig.~\ref{fig:T_N} and Fig.~\ref{fig:n_T}, the values of $n_T$ is much smaller than $T_N$, and consequently, the term $T_N + n_T$ is approximately equal to $T_N$. Thus, the effect of $\mathcal{P}_F$ and $\mathcal{P}_M$ on the bound of $\lambda_A$ is less as compared to $p_{loc}$.

\begin{figure*}[!t]
\centering
\subfloat[]
{\includegraphics[width = 0.25\textwidth, height = 0.225\textwidth]{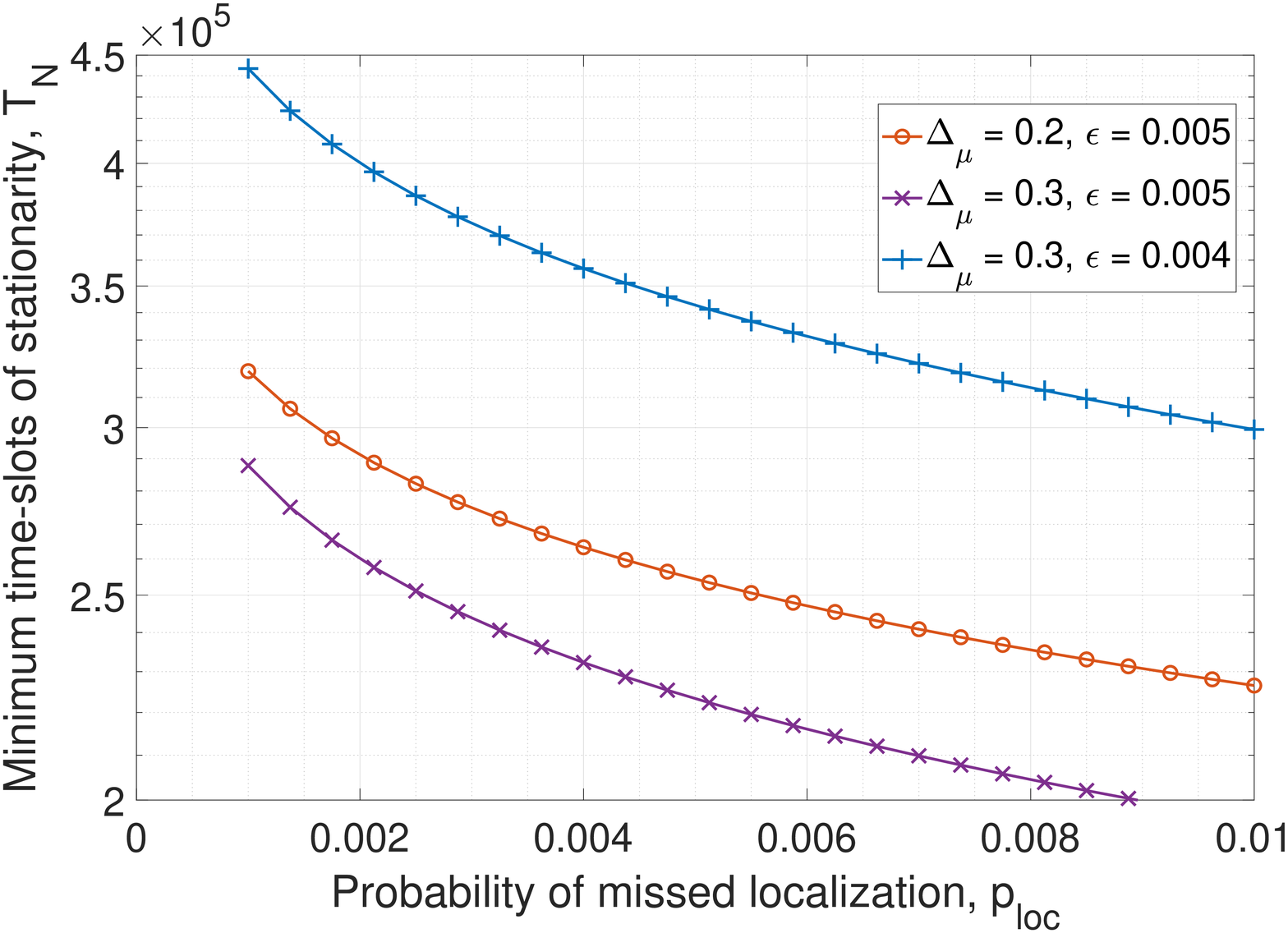}
\label{fig:T_N}}
\subfloat[]
{\includegraphics[width = 0.25\textwidth , height = 0.225\textwidth]{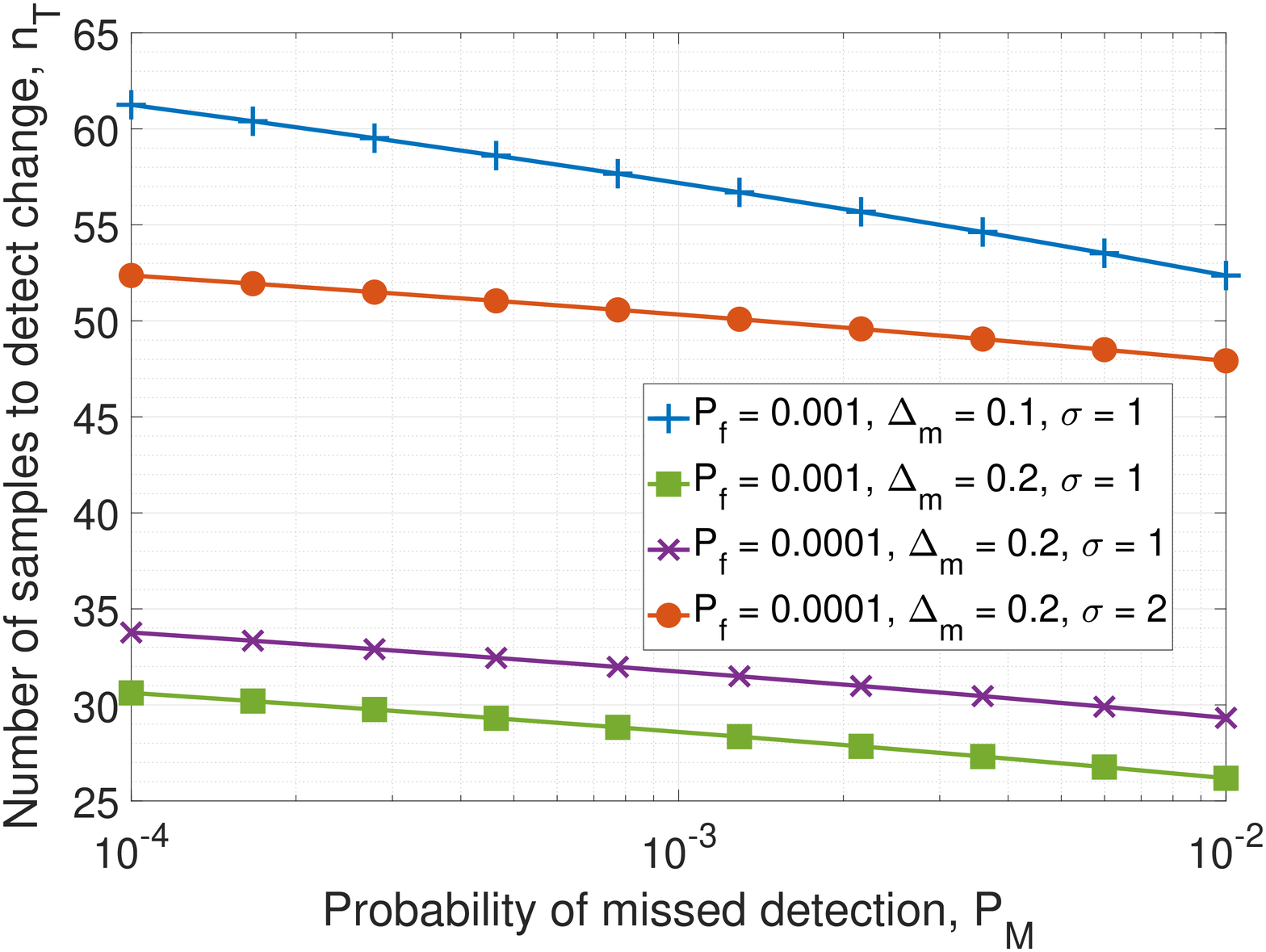}
\label{fig:n_T}}
\subfloat[]
{\includegraphics[width = 0.25\textwidth, height = 0.225\textwidth]{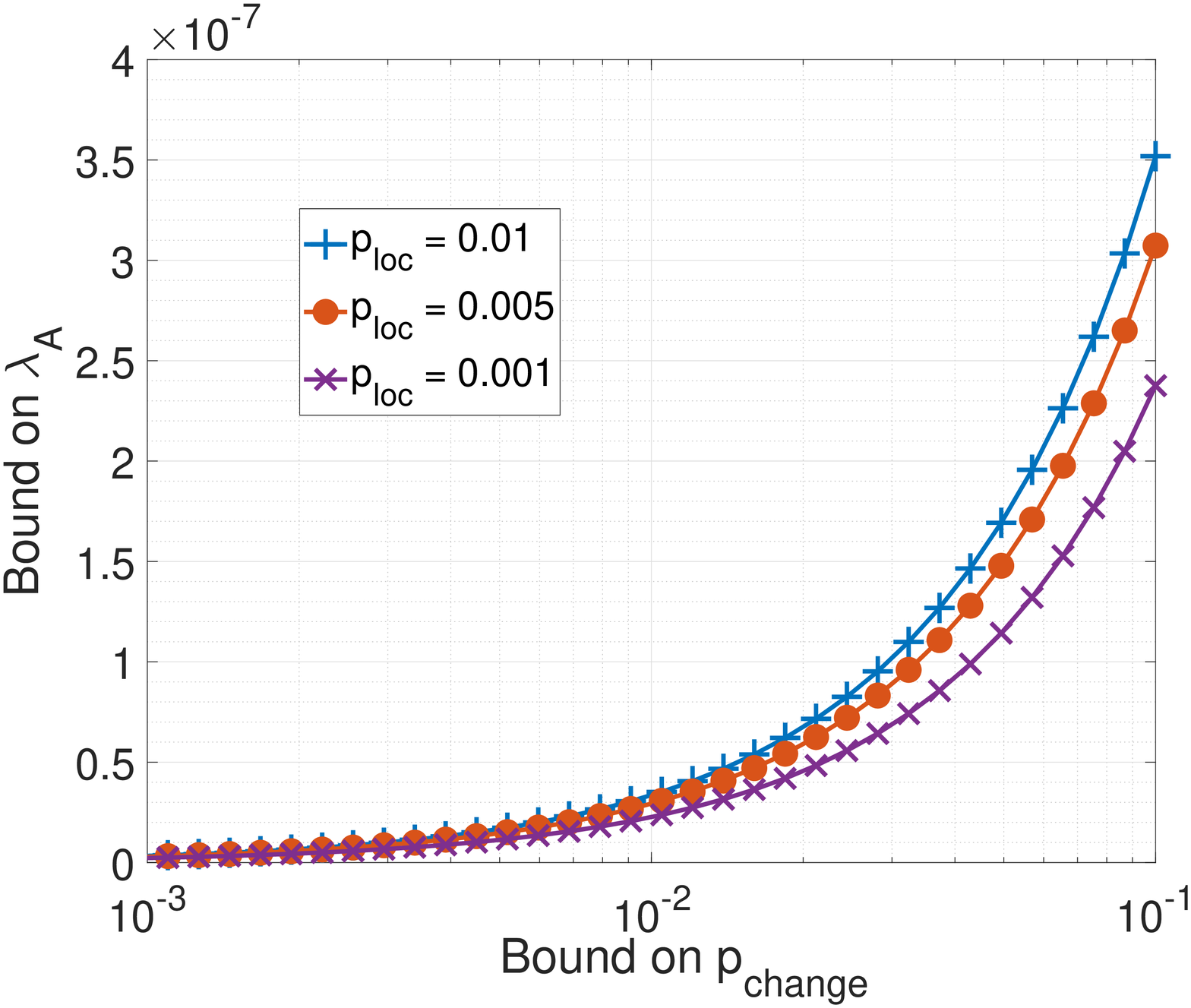}
\label{fig:lambda_A}}
\subfloat[]
{\includegraphics[width = 0.25\textwidth, height = 0.225\textwidth]{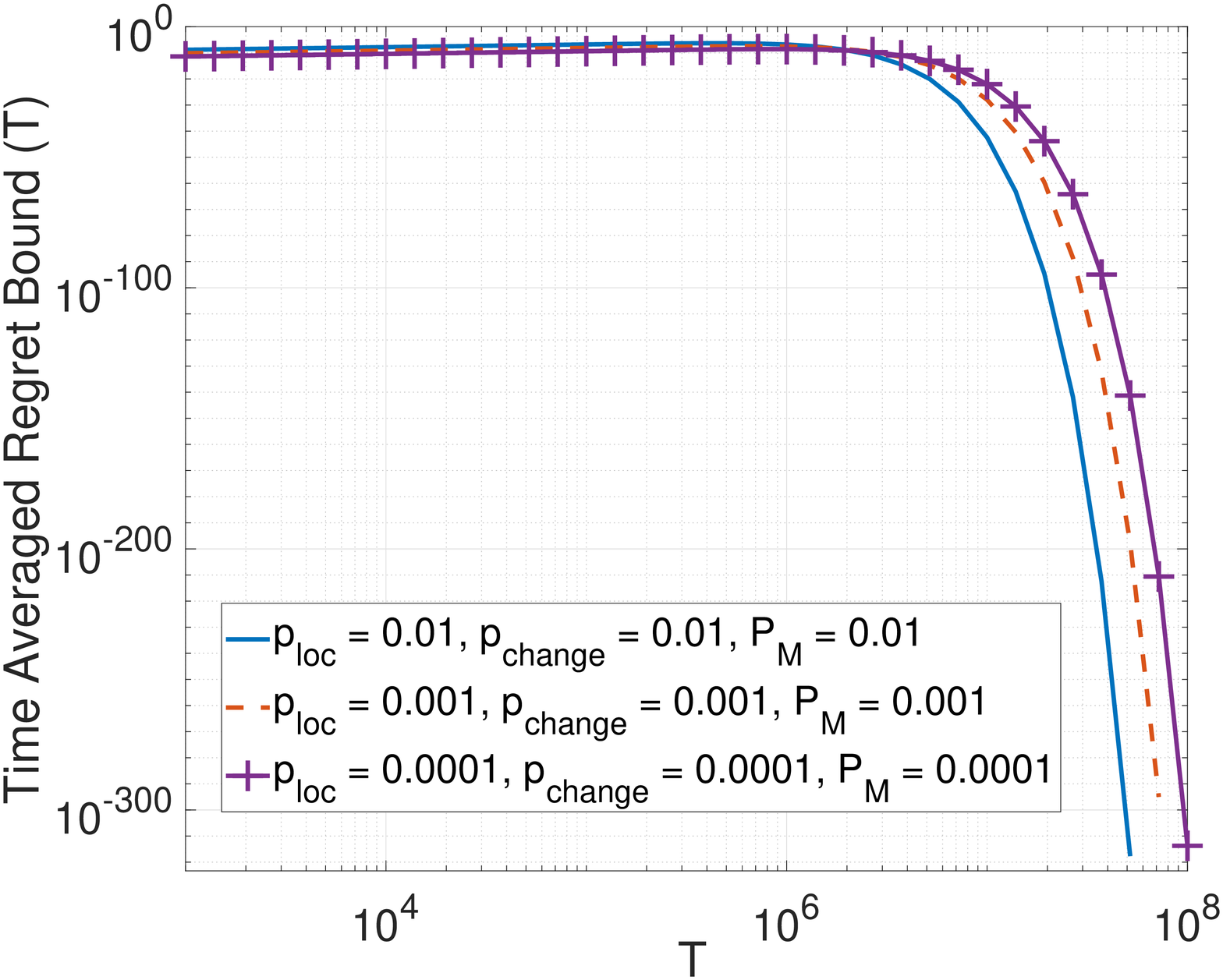}
\label{fig:regret}
}
\caption{(a) Minimum number of time-slots for which the \ac{MAB} framework should be stationary for a given $p_{loc}$, $\Delta_\mu$, and $\epsilon$, (b) Minimum number of time-slots required to detect a change for a given $\mathcal{P}_M$ and $\mathcal{P}_F$, (c) Maximum value of $\lambda_A$ for given $p_{change}$ and $p_{loc}$, (d) Bound on time averaged regret.}
\label{fig:SUBandRAT} 
\end{figure*}



Finally, given the framework developed so far, we derive the regret bound of the \ac{TS-CD} algorithm for arbitrary values of $p_{loc}, p_{change}$, and $\mathcal{P}_M$. We start with the result of Agrawal and Goyal~\cite{agrawal2012analysis} which states that the regret of a stationary \ac{TS} algorithms for the two-armed bandit framework is bounded by the order of $\mathcal{O}(\ln(T))$. Based on this, we derive the following result.
\begin{theorem}
For the two-armed non-stationary bandit problem, with a probability 
\begin{align}
    p_{tot} =  \left(1 - p_{loc}\right)\left(1 - p_{change}\right)\left(1 - \mathcal{P}_M\right),
\end{align}
 the \ac{TS-CD} algorithm has expected regret bound:
\begin{align}
    \mathbb{E}\left[\mathcal{R}(T)\right]  \leq \mathcal{O}\left(\ln\left(T_N\right) \lambda_A T \left[\frac{\Gamma\left(\frac{T}{T_N + n_T}, \lambda_A T\right)}{\Gamma\left(\frac{T}{T_N + n_T}\right)}\right]\right),
\end{align}
in time $T$. Thus, the time expected regret is asymptotically bounded, i.e.,
$
    \lim_{T \to \infty} \frac{1}{T}\mathbb{E}\left[\mathcal{R}(T)\right] = 0.  
$
\end{theorem}

\begin{proof}
In the \ac{TS-CD} framework, each time a change occurs, the player follows the classical \ac{TS} algorithm for $T_N$ time-slots, thereby incurring a regret of $\ln(T_N)$. Thereafter, under the assumption of well-localization of the mean rewards of the arms, the best-arm (during that stationary regime) is played repeatedly until the change is detected. Hence, in a period of $T$ time-slots, the total regret is given by $\eta_T \ln(T_N)$, where $\eta_T$ is the number of changes within time $T$. Let us recall that $\eta_T$ is a Poisson distributed random variable. Furthermore, we recall our assumption that the bandit framework remains constant for at least $T_N$ time-slots after a change and that we need $n_T$ samples to detect the change. Hence, the maximum number of changes in a time-window $T$ is bounded by $\frac{T}{T_N + n_T}$. Accordingly, we bound the regret as:
\begin{align}
  & \mathbb{E}\left[\mathcal{R}(T)\right] \leq  \mathcal{O}\left(\mathbb{E}_{\eta_T}\left[\eta_T \ln\left(T_N\right)\cdot \mathds{1}\left(\eta_T \leq \frac{T}{T_N + n_T}\right)\right]\right),   \nonumber \\
  &= \mathcal{O}\left(\ln\left(T_N\right) \sum_{k = 0}^{\frac{T}{T_N + n_T}} k \ln\left(T_N\right) \exp\left(-\lambda_A T\right)\frac{(\lambda_AT)^k}{k!}\right), \nonumber \\
  & = \mathcal{O}\left(\ln\left(T_N\right) \lambda_A T \left[\frac{\Gamma\left(\frac{T}{T_N + n_T}, \lambda_A T\right)}{\Gamma\left(\frac{T}{T_N + n_T}\right)}\right]\right). \nonumber 
\end{align}
\end{proof}
It is evident that $\lim_{T \to \infty} \frac{1}{T}\mathbb{E}\left[\mathcal{R}(T)\right] = 0$. In Fig.~\ref{fig:regret}, we plot the value of time-averaged regret with respect to $T$ for different values of $p_{loc}$, $p_{change}$, and $\mathcal{P}_M$. We see that as we change the values of the probability bounds from 0.01\% to 1\%, we see that the time-averaged regret bound decreases to 0 by $10^8$ samples. In the next section, we employ the proposed \ac{TS-CD} algorithm to facilitate dynamic \ac{RAT} selection in the edge of a wireless network.

\vspace{-0.3cm}
\section{Case-Study: {RAT Selection} in the Edge}
\label{sec:CS}
The future wireless applications will be characterized by a tremendous increase in demand for data-rates. Among other enabling technologies, transmission in high-frequency ranges, especially in the \ac{mm-wave} spectrum is a promising solution. However, \ac{mm-wave} transmissions suffer from several limitations, such as detrimental path-loss and high sensitivity to blockages~\cite{white}. Consequently, it is evident that the first generation of \ac{mm-wave} \ac{AP} deployment must necessarily be complemented by the existing cellular architecture. Additionally, due to the environment dynamics, such as human and vehicular blockages, the \ac{UE} association to the different \acp{AP} and to different available frequency bands need to be dynamic. In case of \ac{mMTC} and \ac{URLLC} applications, considering the latency constraints and the overheads involved for connecting thousands of devices to the internet, such association must necessarily be decentralized. 

In this section, we investigate a band-switching scheme modeled as the two-armed bandit problem, and study the efficacy of the \ac{TS-CD} algorithm. {In particular, we explore the performance of a blind RAT selection policy, in which, the BS instructs the UE to band switch to a different band without any need for a measurement gap~\cite{mismar2019deep}.}
First, let us discuss the system model under consideration. {The case with higher number of candidate RATs, e.g., multiple sub-6GHz bands, mm-wave and visible light communication etc. is an open problem as far as a theoretical analysis is concerned. The assumption of the dual RAT architecture is motivated by current studies, e.g., see~\cite{5783993} which provides an overview on such dual-RAT architectures that can be used to transmit control and data signals, respectively, at sub-6GHz and millimeter wave frequency bands. The important interplay between the mm-wave band and the sub-6GHz band in a dual RAT architecture is also highlighted in the European project mmMagic~\cite{mmMagic} and the work by Kangas {\it et al.}~\cite{kangas2013angle}.}

{\bf System Model: }
We consider a wireless network consisting of \acp{AP} on the two-dimensional Euclidean plane. The location of the \acp{AP} are modelled as points of a homogeneous \ac{PPP} $\phi$, with intensity $\lambda$. Without loss of generality, we perform our analysis from the perspective of the typical {pedestrian} user located at the origin, and the typical user connects to the \ac{AP} with the strongest downlink received power. The \acp{AP} are assumed to operate in two \acp{RAT}, sub-6GHz band and the \ac{mm-wave} band to provide ad-hoc coverage and enhanced data-rates~\cite{ghatak2018coverage}. To simplify the notation, let us denote the \ac{RAT} with $r$, where $r \in \{m,s\}$ stands for \ac{mm-wave} and sub-6GHz, respectively. Similarly, let us denote the visibility state by $v \in \{L,N\}$ for \ac{LOS} and \ac{NLOS}, respectively. We assume that the received power at the typical user from a \ac{AP} at a distance $x$ from the user is given by $K_rP_rx^{-\alpha_{rv}}$, where $\alpha_{rv}$ is the path-loss exponent for \ac{RAT} $r$ and visibility state $v$, $K_r$ is the path-loss constant for \ac{RAT} $r$, and $P_r$ is the transmit power from the \ac{AP} in \ac{RAT} $r$. The noise power in \ac{RAT} $r$ is denoted by $\sigma_{N,r}^2$. 
In case of mm-wave operations, the {received powers take advantage of} the directional antenna gain of the transmitter and the receiver. {The user and the serving BS are assumed to be aligned, whereas the interfering BSs are randomly oriented with respect to the typical user.}
{Here, we assume a tractable model, where the product of the transmitter and receiver antenna gains, $G$, takes on the values $a_k$ with probabilities $b_k$ as given in Table 1 of \cite{bai2015coverage}. Let the maximum value of $G$ be $G_0$. }

For a given \ac{AP}, the channel visibility state ($v$) in sub-6GHz is assumed to be the same as that in \ac{mm-wave}. {Note that in general, the propagation characteristics in the different bands get affected differently on account of user mobility. However, in this study, we assume that the transitions in both the bands occur simultaneously. This is because, the probability of a signal to be blocked mainly depends on the  blockage process, which is independent of the carrier frequency, e.g., see~\cite{bai2014analysis} for a complete statistical analysis of the same.} Due to the blockages, from the perspective of the typical user, $\phi$ can be further categorized into either \ac{LOS} or \ac{NLOS} processes: {$\phi_{L}$ and $\phi_{N}$, respectively}. The intensity of these modified processes are given by $p(x)\lambda$ and $(1 - p(x))\lambda$, respectively, where $p(x)$ is the probability of a \ac{AP} at a distance $x$ to be in \ac{LOS} with respect to the typical user. For tractability we assume the following \ac{LOS} function~\cite{ghatak2018coverage}: $p(x) = 1; x \leq d$ and $ 0; x > d$. That is, any \ac{AP} within a distance $d$ from the user is assumed to be in \ac{LOS}, and any \ac{AP} beyond a distance $d$ from the user is assumed to be in \ac{NLOS}, where $d$ is the \ac{LOS} ball radii~\cite{ghatak2018coverage}. To study the \ac{SINR} performance, first, the {path-loss processes are reformulated} as one dimensional processes, $\phi'_{vr} = \{\xi_{vr,i}: \xi_{vr,i} = \frac{||x_i||^{\alpha_{vr}}}{K_{r}P_r} , x_i \in \phi_{v}\}$, $v\in \{L,N\}$, $r\in \{s, m\}$. The processes $\phi'_{vr}$ are non-homogeneous with intensity measures derived in \cite{ghatak2018coverage}.


{\bf Characterization of the Mean Rewards: }
Next, we discuss the band-switching scheme in the context of the \ac{MAB} framework. In this scenario, arm 1 represents the sub-6GHz transmission and arm 2 corresponds to \ac{mm-wave} transmission. We assume that the \ac{AP} to user link transition from LOS to NLOS state at unknown time instants, e.g., due to the mobility of the users, resulting in the communication link being obstructed by buildings. {For the experiments, we select $\lambda_C = 0.0005$ (i.e., on an average a change every 2e3 time-steps). This is consistent with stationary periods assumed in the literature, e.g., see~\cite{yu2009piecewise, cao2019nearly}. Let us recall here that each time-step consists of one play of the arm, which occurs at every sub-frame. In the flexible frame structure offered by 5G, we assume a sub-frame duration of 1 ms, which corresponds to the assumption of a stationary regime of about 2 seconds. Now, since our application caters to the outdoor pedestrian users, which generally move with slow speeds, we assume that the stationary regime considered in the paper holds. In particular, for the LOS regime, due to the massive antenna gain of the mm-wave transmissions, we have $\Delta_\mu \geq 0.3$. Additionally, since $n_T << T_N$, the effect of $\Delta_m$ is limited on the required stationary duration. Then, for an $\epsilon$ of 0.01, the stationary regime for the events $E_1$, $E_2$, and $E_3$ to not occur is less than 2 seconds, with $p_{loc} = 0.01$. This validates our assumption and the applicability of \ac{TS-CD}.} Let the rewards of the arms be represented by the \ac{SINR} coverage probability\footnote{The \ac{SINR} coverage probability is the probability that the typical user receives an \ac{SINR} greater than a threshold. Ergodically this represents the fraction of the users in the network under coverage.} of the user for a threshold $\gamma$. Consequently, the mean of each arm follows a two-state Markov model, based on the visibility state. For arm $a_i$, corresponding to \ac{RAT} $r$, the rewards change in the manner: $\mu_{r,{L}} \to \mu_{r,{N}} \to \mu_{r,{L}} \to \cdots$. It must be noted that a transition in one arm coincides with the transition in the other arm, since, the visibility state is the same for both the bands.

\begin{lemma} The \ac{SINR} coverage probability, given that the user is receiving services in the sub-6GHz band from an \ac{AP} at a distance $x$, being in visibility state $v$, is given by: 
\begin{align}
\mu_{s,v} =  \mathbb{P}_{Cvs}(\gamma) = \exp{\left(-\gamma\cdot \sigma_{N,s}^2 \cdot x- \sum_{v'} A_{v'}(\gamma,x)\right)}\nonumber,
\end{align}
where,
$
A_{v'}= \int\limits_{x}^\infty\frac{\gamma x}{y + \gamma x}  \Lambda'_{v's}(dy)\nonumber, \quad \forall \; v' \in \{L,N\}.
$
\label{lem:sub6}
\end{lemma}
\begin{proof}
        The proof is similar to that in~\cite{ghatak2018coverage}.
\end{proof}

\begin{lemma} The \ac{SINR} coverage probability, given that the user is receiving services in the mm-wave from an \ac{AP} located at distance $x$, being in visibility state $v$, is given by:
\begin{align}
&\mu_{m,v} = \mathbb{P}_{Cvm}(\gamma) \nonumber \\
&=  \exp\left(-\frac{\gamma \cdot x \cdot \sigma_{N,m}^2}{G_0} - B_1(\gamma,x) - B_2(\gamma,x)\right), \\ 
  &\mbox{with\quad} B_1(\gamma,x) = \sum_{k=1}^4 \left(-b_k \int\limits_x^\infty\left(\frac{a_k\gamma x}{y + a_k\gamma x} \Lambda'_{vm}(dy) \right)\right), \nonumber \\
  &\mbox{and,\quad}  B_2(\gamma,x) = \sum_{k=1}^4  \left(-b_k \int\limits_x^\infty\left(\frac{a_k\gamma x}{y + a_k\gamma x} \Lambda'_{v'm}(dy) \right)\right) \nonumber .
	\end{align}
    \end{lemma}
\begin{proof}
		The proof follows in a similar way to that of Lemma~\ref{lem:sub6}.
\end{proof}
{Before proceeding to the numerical results section, it is important to note that in general, the RAT-selection penalty is non-negligible and is associated with not only a receiver reconfiguration overhead, but also a non-zero switching delay, which we ignore in the experiment due to the blind switching policy, and to be consistent with the Algorithm presented in Section 2.2. Nevertheless, the RAT configuration delay and the overhead, e.g., for legacy switching schemes can be integrated into the proposed \ac{TS-CD} framework by introducing a cost or penalty to switch arms. However, a complete analysis of the switching costs on the proposed algorithm, and the corresponding regret bounds will be treated in a future work.}

{\bf Simulation Setup and Results: } In Fig.~\ref{fig:my_label} we compare {the following 6 algorithms in the context of RAT switching. Actively adaptive algorithms: 1)  PHT-UCB, which uses the Page-Hinkely statistic to detect the change~\cite{liu2018change}, 2) TS-CD, which is the proposed algorithm in our work. Passively adaptive algorithms: 3) REXP3~\cite{besbes2019optimal}, 4) Discounted TS (DTS)~\cite{raj2017taming}, 5) Discounted UCB (D-UCB)~\cite{garivier2008upper}, and 6) Max-power association (called {\it Fixed}), which is a classical association rule in legacy wireless networks. } {We clearly see that the \ac{TS-CD} algorithm not only outperforms the static association rule based on maximum received power, but also performs better than the other contending passively adaptive and actively adaptive bandit algorithms.}
{
In addition, we plot the respective variances of the different algorithms. The static association rule is shown to have the maximum variance among all the contending algorithms. This is mainly driven by the assumed variance of the rewards of a particular arm. In comparison, the passively adaptive algorithms have lower variance. However, in the initial stages, i.e., when the number of plays is small, a change in the environment results in an increase in the normalized regret of the passively adaptive algorithms.}

{
On the contrary, the actively adaptive algorithms decrease the variance even more, establishing their higher stability as compared to the passively adaptive algorithms. Second, we show an example of the delay of change detection in the zoomed portion of the figure. Since in our case study, the difference between the \ac{SINR} coverage probability of mm-wave RAT in \ac{LOS} and \ac{NLOS} is considerably large, we have a very high $\Delta_\mu$ and $\Delta_m$. This results in a significantly quick detection of the change as compared to the recently investigated PHT~\cite{liu2018change}.
}
\begin{figure}
    \centering
    \includegraphics[width = 0.4\textwidth]{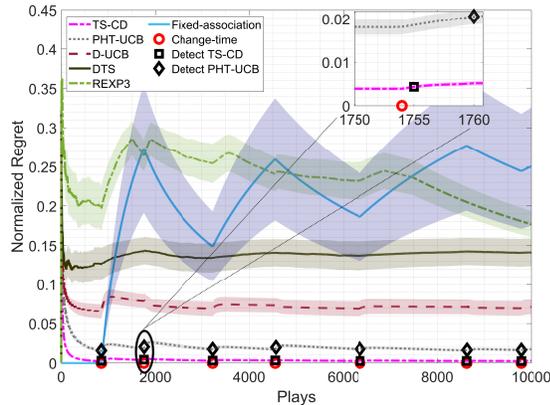}
    \caption{Time-averaged regret for different strategies.}
    \label{fig:my_label}
\end{figure}

Currently, we are investigating more efficient change detection policies based on goodness of fit tests. {Additionally, the sensitivity analysis of the algorithm with respect to the different system parameters of the example scenario concerned is a key direction of research, which we will address in a future work}

\vspace{-0.5cm}
\section{Conclusion}
\label{sec:Con}
In this paper, we have investigated a change-detection based Thompson Sampling algorithm, named TS-CD, to keep track of the dynamic environment in the non-stationary two-armed bandit problem. We have derived the lower bound on the stationary regime time-window for TS-CD to efficiently detect the changes when they occur. Finally, we show that for given bounds on the frequency of changes, the proposed TS-CD algorithm reaches asymptotic optimality. To test the efficacy of the algorithm, we employ it in the RAT selection problem in a wireless network edge. We have shown that TS-CD not only outperforms the classical max-power band selection scheme, but also, it outperforms other bandit algorithms designed for dynamic environments.

\vspace{-0.2cm}
\bibliographystyle{IEEEtran}
\bibliography{refer.bib}

\begin{thebibliography}{10}
\providecommand{\url}[1]{#1}
\def\UrlFont{\rmfamily}
\providecommand{\newblock}{\relax}
\providecommand{\bibinfo}[2]{#2}
\providecommand\BIBentrySTDinterwordspacing{\spaceskip=0pt\relax}
\providecommand\BIBentryALTinterwordstretchfactor{4}
\providecommand\BIBentryALTinterwordspacing{\spaceskip=\fontdimen2\font plus
\BIBentryALTinterwordstretchfactor\fontdimen3\font minus
  \fontdimen4\font\relax}
\providecommand\BIBforeignlanguage[2]{{%
\expandafter\ifx\csname l@#1\endcsname\relax
\typeout{** WARNING: IEEEtran.bst: No hyphenation pattern has been}%
\typeout{** loaded for the language `#1'. Using the pattern for}%
\typeout{** the default language instead.}%
\else
\language=\csname l@#1\endcsname
\fi
#2}}

\bibitem{villar2015multi}
S.~S. Villar \emph{et~al.}, ``Multi-armed bandit models for the optimal design
  of clinical trials: benefits and challenges,'' \emph{Statistical science: a
  review journal of the Institute of Mathematical Statistics}, vol.~30, no.~2,
  p. 199, 2015.

\bibitem{li2016collaborative}
S.~Li \emph{et~al.}, ``Collaborative filtering bandits,'' in \emph{Proceedings
  of the 39th International ACM SIGIR conference on Research and Development in
  Information Retrieval}, 2016, pp. 539--548.

\bibitem{buccapatnam2017reward}
S.~Buccapatnam, \emph{et~al.}, ``Reward maximization under uncertainty:
  Leveraging side-observations on networks,'' \emph{The Journal of Machine
  Learning Research}, vol.~18, no.~1, pp. 7947--7980, 2017.

\bibitem{rahman2019beam}
A.~U. Rahman and G.~Ghatak, ``A beam-switching scheme for resilient mm-wave
  communications with dynamic link blockages,'' in \emph{IEEE WiOpt}, 2019.

\bibitem{contal2013parallel}
E.~Contal, \emph{et~al.}, ``Parallel gaussian process optimization with upper
  confidence bound and pure exploration,'' in \emph{Joint European Conference
  on Machine Learning and Knowledge Discovery in Databases}.\hskip 1em plus
  0.5em minus 0.4em\relax Springer, 2013, pp. 225--240.

\bibitem{thompson1933likelihood}
W.~R. Thompson, ``On the likelihood that one unknown probability exceeds
  another in view of the evidence of two samples,'' \emph{Biometrika}, vol.~25,
  no. 3/4, pp. 285--294, 1933.

\bibitem{granmo2010solving}
O.-C. Granmo, ``Solving two-armed bernoulli bandit problems using a bayesian
  learning automaton,'' \emph{International Journal of Intelligent Computing
  and Cybernetics}, 2010.

\bibitem{chapelle2011empirical}
O.~Chapelle and L.~Li, ``An empirical evaluation of thompson sampling,'' in
  \emph{Advances in neural information processing systems}, 2011, pp.
  2249--2257.

\bibitem{agrawal2012analysis}
S.~Agrawal and N.~Goyal, ``Analysis of thompson sampling for the multi-armed
  bandit problem,'' in \emph{Conference on learning theory}, 2012, pp. 39--1.

\bibitem{garivier2008upper}
A.~Garivier and E.~Moulines, ``On upper-confidence bound policies for
  non-stationary bandit problems,'' \emph{arXiv preprint arXiv:0805.3415},
  2008.

\bibitem{gupta2011thompson}
N.~Gupta \emph{et~al.}, ``Thompson sampling for dynamic multi-armed bandits,''
  in \emph{Machine Learning and Applications and Workshops (ICMLA), 2011 10th
  International Conference on}, vol.~1.\hskip 1em plus 0.5em minus 0.4em\relax
  IEEE, 2011, pp. 484--489.

\bibitem{raj2017taming}
V.~Raj and S.~Kalyani, ``Taming non-stationary bandits: A bayesian approach,''
  \emph{arXiv preprint arXiv:1707.09727}, 2017.

\bibitem{besbes2014stochastic}
O.~Besbes \emph{et~al.}, ``Stochastic multi-armed-bandit problem with
  non-stationary rewards,'' in \emph{Advances in neural information processing
  systems}, 2014, pp. 199--207.

\bibitem{besbes2019optimal}
------, ``Optimal exploration-exploitation in a multi-armed bandit problem with
  non-stationary rewards,'' \emph{Stochastic Systems}, vol.~9, no.~4, pp.
  319--337, 2019.

\bibitem{hartland2006multi}
C.~Hartland, \emph{et~al.}, ``Multi-armed bandit, dynamic environments and
  meta-bandits,'' 2006.

\bibitem{srivastava2014surveillance}
V.~Srivastava, P.~Reverdy, and N.~E. Leonard, ``Surveillance in an abruptly
  changing world via multiarmed bandits,'' in \emph{53rd IEEE Conference on
  Decision and Control}.\hskip 1em plus 0.5em minus 0.4em\relax IEEE, 2014, pp.
  692--697.

\bibitem{yu2009piecewise}
J.~Y. Yu and S.~Mannor, ``Piecewise-stationary bandit problems with side
  observations,'' in \emph{Proceedings of the 26th annual international
  conference on machine learning}, 2009, pp. 1177--1184.

\bibitem{cao2019nearly}
Y.~Cao, \emph{et~al.}, ``Nearly optimal adaptive procedure with change
  detection for piecewise-stationary bandit,'' in \emph{The 22nd International
  Conference on Artificial Intelligence and Statistics}, 2019, pp. 418--427.

\bibitem{mellor2013thompson}
J.~Mellor and J.~Shapiro, ``Thompson sampling in switching environments with
  bayesian online change detection,'' in \emph{Artificial Intelligence and
  Statistics}, 2013, pp. 442--450.

\bibitem{white}
{White Paper}, ``{5G Channel Model for bands up to 100 GHz},''
  \emph{http://www.5gworkshops.com/5gcm.html}, Dec. 6. 2015.

\bibitem{mismar2019deep}
F.~B. Mismar, \emph{et~al.}, ``Deep learning predictive band switching in
  wireless networks,'' \emph{arXiv preprint arXiv:1910.05305}, 2019.

\bibitem{5783993}
D.~Aziz, \emph{et~al.}, ``{Architecture Approaches for 5G Millimetre Wave
  Access Assisted by 5G Low-Band Using Multi-Connectivity},'' in \emph{2016
  IEEE Globecom Workshops (GC Wkshps)}, Dec 2016, pp. 1--6.

\bibitem{mmMagic}
{H2020-ICT-671650 mmMAGIC}, ``{D3.1: Initial concepts on 5G architecture and
  integration},'' \emph{{Available Online at https://5g-mmmagic.eu/}}, Mar.
  2016.

\bibitem{kangas2013angle}
A.~Kangas and T.~Wigren, ``Angle of arrival localization in lte using mimo
  pre-coder index feedback,'' \emph{IEEE Communications Letters}, vol.~17,
  no.~8, pp. 1584--1587, 2013.

\bibitem{ghatak2018coverage}
G.~Ghatak \emph{et~al.}, ``Coverage analysis and load balancing in hetnets with
  millimeter wave multi-rat small cells,'' \emph{IEEE Transactions on Wireless
  Communications}, vol.~17, no.~5, pp. 3154--3169, 2018.

\bibitem{bai2015coverage}
T.~Bai and R.~W. Heath, ``{Coverage and Rate Analysis for Millimeter-Wave
  Cellular Networks},'' \emph{IEEE Trans. Wireless Commun.}, vol.~14, no.~2,
  pp. 1100--1114, 2015.

\bibitem{bai2014analysis}
T.~Bai, R.~Vaze, and R.~W. Heath, ``{Analysis of Blockage Effects on Urban
  Cellular Networks},'' \emph{IEEE Trans. Wireless Commun.}, vol.~13, no.~9,
  pp. 5070--5083, 2014.

\bibitem{liu2018change}
F.~Liu, J.~Lee, and N.~Shroff, ``A change-detection based framework for
  piecewise-stationary multi-armed bandit problem,'' in \emph{Thirty-Second
  AAAI Conference on Artificial Intelligence}, 2018.

\end{thebibliography}

\end{document}